\documentclass{article}

\usepackage[preprint]{neurips_2026}

\usepackage[utf8]{inputenc}
\usepackage[T1]{fontenc}
\usepackage{hyperref}
\usepackage{url}
\usepackage{booktabs}
\usepackage{amsfonts}
\usepackage{amsmath}
\usepackage{amssymb}
\usepackage{amsthm}
\usepackage{mathtools}

\newtheorem{theorem}{Theorem}
\newtheorem{proposition}{Proposition}
\usepackage{nicefrac}
\usepackage{microtype}
\usepackage{graphicx}
\usepackage{xcolor}
\usepackage{multirow}
\usepackage{subcaption}

\newcommand{\Vstar}{V^{*}}
\newcommand{\Qsafe}{Q^{*}_{\text{safe}}}
\newcommand{\Qrisky}{Q^{*}_{\text{risky}}}
\newcommand{\lambdastar}{\lambda^{*}}
\newcommand{\SCAT}{S_{\text{cat}}}
\newcommand{\VCAT}{V_{\text{cat}}}
\newcommand{\Dw}{\Delta_{w}}
\newcommand{\Dl}{\Delta_{\ell}}

\newcommand{\Ds}{\Delta_{s}}
\newcommand{\EE}{\mathbb{E}}

\title{Prospect-Theory Behavior from Bellman Optimality\\ in MDPs with Catastrophic States}

\author{%
  Yujiao Chen\\
  Massachusetts Institute of Technology\\
  Cambridge, MA 02139, USA\\
  \texttt{yujiaoch@mit.edu}
}

\begin{document}

\maketitle

\begin{abstract}
We study risk-neutral control in Markov decision processes with an absorbing catastrophic state. Even though rewards are linear and the agent has no utility curvature, probability weighting, or framing dependence, standard Bellman optimality produces three prospect-theory-like signatures: an S-shaped value-function profile (convex near catastrophe, concave in the far field), an endogenous loss-sensitivity coefficient $\lambdastar(S) > 1$, and a reflection-effect policy reversal. Across 495 parameter configurations, the optimal policy plays safe near catastrophe in positive-drift (growth) regimes despite the risky action's higher immediate expected value, and plays risky near catastrophe in negative-drift (decline) regimes despite the safe action's lower immediate expected loss. We derive a closed-form expression for the asymptotic loss-aversion plateau $\bar{\lambda} = \lim_{S \to \infty} \lambdastar(S)$ that depends only on win probability $p$, payoff asymmetry $r = |\Dl/\Dw|$, and discount factor $\beta$, and matches numerical solutions to $R^2 = 0.999$. The mechanism does not require asymmetric payoffs. Across a sweep of $(p, \beta)$ at three asymmetry levels, the asymmetry share of $\bar{\lambda}$ above unity has median $4.6\%$ at $r = 1.25$ and rises to $13.9\%$ at $r = 2$, with the boundary contribution exceeding the asymmetry contribution in every $(p, \beta, r)$ cell tested. The phenomena persist under tabular Q-learning---a model-free agent reproduces $\Vstar$ at correlation $0.98$ in growth and $1.00$ in decline---and under stochastic transitions with Gaussian, heavy-tailed Student-$t_3$, and asymmetric skew-normal noise up to $50\%$ of the step size, where the asymptotic plateau tracks the closed-form prediction within $0.41\%$ for safe-channel noise and within $9.6\%$ for risky-channel or both-channel noise. These results identify absorbing failure states as a sufficient structural mechanism for prospect-theory-like behavior under optimal control.
\end{abstract}

\section{Introduction}

Sequential decisions made near absorbing failure states can produce strong asymmetries between gains and losses even for risk-neutral agents. We formalize this observation in a Markov decision process with linear rewards and an absorbing catastrophe boundary, and show that standard Bellman optimality alone generates three prospect-theory-like signatures: an S-shaped value-function profile (convex near catastrophe, concave in the far field), an endogenous loss-sensitivity coefficient, and a reflection-effect policy reversal. The mechanism is purely dynamic: continuation values make downside risk disproportionately consequential near irreversible failure, even when the agent has no utility curvature, probability weighting, or framing dependence.

Prospect theory \citep{kahneman1979prospect, tversky1992advances} documents three robust patterns in human decision-making under risk: an S-shaped value function (concave for gains, convex for losses); loss aversion ($\lambda \approx 2.25$, so that losses loom roughly twice as large as equivalent gains); and the reflection effect (risk aversion in gains, risk seeking in losses). These patterns are typically interpreted as evidence for nonlinear preferences or distorted probability perception. We do not argue that the present mechanism explains human behavior, nor do we claim to reproduce prospect theory in full; probability weighting, framing, and reference dependence remain outside the model. Our contribution is structural and normative---a class of environments in which prospect-theory-like patterns are optimal under standard Bellman objectives---and is meant to sit alongside the psychological accounts rather than displace them.

Consider a decision-maker choosing between a safe action (deterministic payoff $\Ds$) and a risky action (win $\Dw$ with probability $p$, lose $\Dl$ with probability $1{-}p$). Two natural regimes arise. In \textbf{growth} ($\EE[\text{risky}] > \Ds > 0$), both actions improve the agent's position, but risky offers higher expected value. In \textbf{decline} ($\EE[\text{risky}] < \Ds < 0$), both actions deteriorate the agent's position, but safe deteriorates less. Local expected-value comparisons therefore favor risky in growth and safe in decline. An absorbing catastrophe boundary, however, can invert these local rankings: the continuation value depends on how close the agent is to irreversible failure, and that proximity can outweigh the one-step expected payoff. Near catastrophe in growth, the agent may optimally lock in smaller gains to avoid a large downside jump; near catastrophe in decline, the safe action only leads to certain ruin more slowly, so optimal play may be to gamble. The same Bellman equation thus produces risk aversion in gains and risk seeking in losses as a consequence of environment structure alone. Catastrophe risk---an effectively absorbing threshold corresponding to outcomes like death, bankruptcy, or exclusion---is the class of environments we focus on.

\subsection{Contributions}

\textbf{1. Three prospect-theory signatures from Bellman optimality.} Across 495 configurations, standard Bellman optimality with linear rewards produces an S-shaped value-function profile in growth, state-dependent loss sensitivity ($\lambdastar(S) > 1$), and a reflection-effect policy reversal across growth and decline (Section~\ref{sec:discovery}).

\textbf{2. The reflection effect arises from continuation values, not preferences.} The Bellman-optimal agent plays safe near catastrophe in growth despite risky's higher \emph{immediate} expected value, and gambles near catastrophe in decline despite safe's lower \emph{immediate} expected loss. The same agent and objective yield opposite risk attitudes when environmental drift changes sign.

\textbf{3. A closed-form formula for asymptotic loss aversion, with the boundary alone sufficient.} Theorem~\ref{thm:closedform} gives $\bar{\lambda} = (z^{-r} - 1)/(1 - z)$ where $z$ solves $\beta[pz + (1{-}p)z^{-r}] = 1$; Proposition~\ref{prop:scale} establishes scale invariance, and Proposition~\ref{prop:lambda1} proves $\bar{\lambda} > 1$ for all $r \geq 1$ in growth. The formula matches simulations at $R^2 = 0.999$. Across a $(p, \beta) \times r$ sweep, the boundary contribution exceeds the asymmetry contribution in every one of $48$ cells tested (Section~\ref{sec:scaling}).

\textbf{4. Robustness to learning and stochasticity.} The theoretical phenomena survive tabular Q-learning (correlation $0.98$ in growth, $1.00$ in decline) and stochastic transitions under Gaussian, heavy-tailed, and asymmetric noise up to $50\%$ of the step size, where the asymptotic plateau deviates from the closed-form prediction by at most $9.6\%$ (Section~\ref{sec:validation}).

\subsection{Related work}

Closest to our work are literatures on sequential decision-making with ruin or absorbing failure. \emph{Risk-sensitive MDPs} \citep{howard1972risk, borkar2002risk} modify the objective by introducing nonlinear utility or exponential transforms; we keep a standard risk-neutral Bellman objective and show that the environment alone induces prospect-theory-like behavior. \emph{Kelly criterion} \citep{kelly1956new} and classical ruin analyses study survival and bet sizing under ruin but do not derive a state-dependent loss-aversion coefficient or the joint appearance of an S-shaped value function and the reflection effect. \emph{Ergodicity economics} \citep{peters2019ergodicity} emphasizes how absorbing losses and multiplicative dynamics alter long-run growth, which is closely related in spirit, but does not characterize the geometry of $\Vstar$ we study here. \emph{Evolutionary accounts} \citep{robson2001biological, hintze2015risk} argue that risk attitudes can be adaptive but do not derive them as exact Bellman-optimal policies. \emph{Reference-dependent models} \citep{koszegi2006model} take loss aversion as primitive; here a prospect-theory-like $\lambdastar(S)$ is endogenous. \emph{Safe RL} \citep{garcia2015comprehensive, altman1999constrained, saunders2018trial} studies absorbing failure states, constraints, and safety guarantees near catastrophe, but its emphasis is typically on constraint satisfaction rather than the geometry of $\Vstar$ and the behavioral signatures it implies.

A growing empirical literature finds that the loss-aversion coefficient varies systematically with the structure of the gamble. Meta-analyses by \citet{brown2024meta} and \citet{yechiam2025loss} report that estimated $\lambda$ depends on the gain/loss range used in elicitation; \citet{walasek2015makeloss} show experimentally that loss aversion can be attenuated or reversed by manipulating the symmetry of the choice set; \citet{yechiam2013losses} and \citet{mukherjee2017magnitude} document attentional and reference-point mechanisms tied to gamble structure rather than to a stable preference parameter; and \citet{gal2018loss} synthesize evidence against a universal $2{:}1$ ratio. Our closed-form formula \eqref{eq:master} speaks directly to these patterns by isolating a boundary-induced contribution that is present even at $r = 1$ and quantifying how intrinsic asymmetry adds to it. To our knowledge, no prior work isolates a structural mechanism that jointly produces the S-shaped value-function profile, endogenous loss aversion, and the reflection effect under standard Bellman optimality, and characterizes that mechanism analytically.

\section{Model}
\label{sec:model}

We study a catastrophe-boundary MDP with state space $\mathcal{S} = \{\SCAT, \SCAT{+}1, \ldots\}$ (the upper bound $\bar{S}$ is a computational truncation; all theoretical statements are for the semi-infinite extension with $\bar{S} = \infty$). Without loss of generality we fix $\SCAT = 0$. Two actions are available: \emph{safe} (deterministic transition $S' = S + \Ds$) and \emph{risky} ($S' = S + \Dw$ with probability $p$, $S' = S + \Dl$ with probability $1{-}p$, where $\Dw > 0 > \Dl$). The state $\SCAT$ is an absorbing catastrophe with $\Vstar(\SCAT) = \VCAT \leq 0$. Rewards are linear in the state increment, the discount factor is $\beta \in (0,1)$, and the agent is risk-neutral. The Bellman optimality equation is
\begin{equation}
\Vstar(S) \;=\; \max\Bigl\{\underbrace{\Ds + \beta\,\Vstar(S + \Ds)}_{\Qsafe(S)},\; \underbrace{p\bigl[\Dw + \beta\,\Vstar(S + \Dw)\bigr] + (1{-}p)\bigl[\Dl + \beta\,\Vstar(S + \Dl)\bigr]}_{\Qrisky(S)}\Bigr\},
\label{eq:bellman}
\end{equation}
with $\Vstar(S) = \VCAT$ for $S \leq \SCAT$. To quantify the asymmetry induced by the value function, we define the \emph{endogenous loss sensitivity} as the ratio of value drops to value gains,
\begin{equation}
\lambdastar(S) \;=\; \frac{|\Vstar(S) - \Vstar(S + \Dl)|}{|\Vstar(S + \Dw) - \Vstar(S)|},
\label{eq:lambda}
\end{equation}
the MDP analogue of the prospect-theory loss-aversion coefficient $\lambda$. We emphasize that $\lambdastar(S)$ is a value-function slope ratio, not a utility ratio, and should not be interpreted as a direct estimate of the psychological parameter $\lambda$; the comparison we draw is structural. Unlike a primitive preference parameter, $\lambdastar(S)$ is state-dependent: it varies across the state space, peaks near the policy boundary in growth environments, and converges in the far field to $\bar{\lambda} = \lim_{S \to \infty} \lambdastar(S) > 1$.

We study two regimes: \textbf{growth} ($\EE[\text{risky}] > \Ds > 0$), where both actions improve the agent's state but risky offers higher immediate expected value, and \textbf{decline} ($\EE[\text{risky}] < \Ds < 0$), where both actions deteriorate the state but safe has lower immediate expected loss. A local one-step comparison that ignores continuation values would therefore favor risky in growth and safe in decline. We show that standard Bellman optimality with an absorbing boundary systematically \emph{overturns} these local rankings near catastrophe. We solve the system by value iteration to tolerance $10^{-10}$ across 495 parameterizations organized into 14 experimental series (Appendix~\ref{app:params}).

\section{Prospect-theory signatures from Bellman optimality}
\label{sec:discovery}

\subsection{Growth environments: S-curve and endogenous loss aversion}
\label{sec:growth}

In growth, near the catastrophe boundary, continuation values make the downside jump disproportionately costly, so the Bellman-optimal agent chooses safe despite the risky action's higher immediate expected return. The resulting $\Vstar(S)$ is S-shaped---convex near catastrophe and concave in the far field (Figure~\ref{fig:growth_example})---a consequence of the Bellman equation rather than an assumption. Three structural features of the S-shape have closed-form explanations: a convex safe zone, a curvature kink at the policy boundary $|\Dl|$, and fading scallops in the far field.

\paragraph{Convexity in the safe zone ($S < |\Dl|$).} Within the safe zone, the agent takes the safe action at every state, so $\Vstar$ satisfies the recurrence $\Vstar(S) = \Ds + \beta\,\Vstar(S + \Ds)$. Telescoping from state $S$ to the safe-zone edge yields
\begin{equation}
\Vstar(S) \;=\; \frac{\Ds(1 - \beta^d)}{1 - \beta} \;+\; \beta^d\,\Vstar(S + d\,\Ds), \qquad d = \left\lceil \frac{|\Dl| - S}{\Ds} \right\rceil,
\label{eq:safe_zone}
\end{equation}
where $d$ is the number of safe steps required to exit the safe zone. (When $\Ds$ divides $|\Dl| - S$---in particular for $\Ds = 1$, as in our figures---$S + d\,\Ds = |\Dl|$ exactly.) Because the risky zone offers far higher growth ($\EE[\text{risky}] \gg \Ds$), we have $\Vstar(S + d\,\Ds) \gg \Ds / (1 - \beta)$, and the dominant term $\beta^d\,\Vstar(S + d\,\Ds)$ is convex in $S$: each step toward the risky zone is exponentially more valuable than the last, since the high-return risky action is one fewer discounted step away.

\paragraph{Curvature kink at $S = |\Dl|$.} At this threshold, the risky action's loss outcome $S + \Dl$ transitions from landing inside the catastrophe region ($S + \Dl \leq \SCAT$ for $S < |\Dl|$) to landing strictly above it ($S + \Dl > \SCAT$ for $S > |\Dl|$). Below $|\Dl|$, each risky loss is catastrophic, creating the steep convex escape ramp. Strictly above $|\Dl|$, the loss outcome has finite value, and diminishing returns from distance dominate---the curve bends concave. The transition is a discrete change in loss consequences manifest as a finite jump in the discrete second difference $\Delta^2 \Vstar(S) = \Vstar(S{+}1) - 2\Vstar(S) + \Vstar(S{-}1)$ at $S = |\Dl|$, where the optimal action switches from safe to risky. We refer to this policy-boundary discontinuity as the S-curve's ``inflection'' for brevity.

\paragraph{Scalloping in the far field.} In the risky zone, $\Vstar(S)$ depends on $\Vstar(S + \Dl)$ through the Bellman equation. As $S$ crosses each multiple $k|\Dl|$, the loss destination $S + \Dl = (k{-}1)|\Dl|$ crosses a boundary where $\Vstar$ has a curvature discontinuity inherited from the previous generation. At $S = k|\Dl|$, catastrophe is reachable in exactly $k$ consecutive losses (probability $(1{-}p)^k$), so each additional $|\Dl|$ of distance adds one more ``loss buffer.'' The amplitude of these scallops decays geometrically: the concavity spike at the $k$-th multiple is attenuated by a factor of $\rho^{|\Dl|} = z^{r}$ per generation, where $z \in (0,1)$ is the unique decaying root of the characteristic equation derived in Section~\ref{sec:scaling}. For the baseline ($p = 0.65$, $\beta = 0.95$, $r = 1.25$), this gives $z^{r} \approx 0.53$, in good agreement with the empirically measured ratios of consecutive $V''$ peaks. The scallops fade into the smooth concave far field.

The S-shape directly induces an endogenous loss asymmetry. The loss sensitivity $\lambdastar(S)$ peaks near the safe-zone boundary and settles to a plateau $\bar{\lambda} > 1$ in the far field (Figure~\ref{fig:growth_example}, center). The numerical value depends on $(p, \beta, r)$; the closed-form Section~\ref{sec:scaling} identifies the boundary-induced contribution that is present even at $r = 1$.

\begin{figure}[t]
\centering
\includegraphics[width=0.95\textwidth]{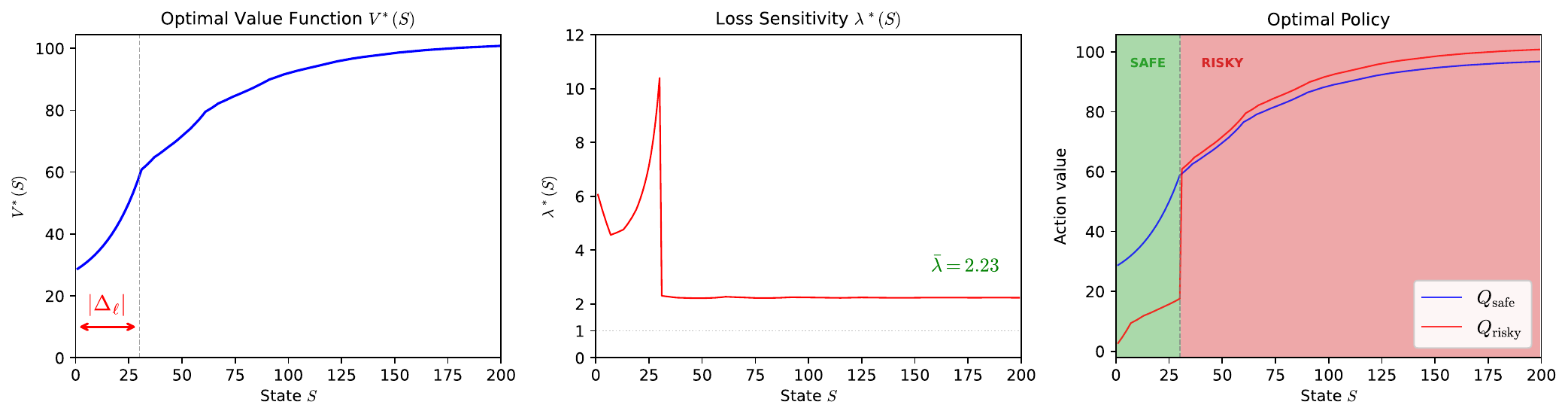}
\caption{\textbf{Growth environment: Bellman-optimal caution near catastrophe} ($\Ds = +1$, $\Dw = +24$, $\Dl = -30$, $p = 0.65$, $\beta = 0.95$, $\VCAT = -100$). \emph{Left:} value-function profile $\Vstar(S)$, convex near catastrophe and concave far away. \emph{Center:} $\lambdastar(S)$ plateaus at $\approx 2.23$. \emph{Right:} the agent voluntarily plays safe for 30 states ($= |\Dl|$) even though the risky action has higher immediate expected value.}
\label{fig:growth_example}
\end{figure}

\subsection{Decline environments: the hail Mary effect}
\label{sec:decline}

In decline ($\EE[\text{risky}] < \Ds < 0$), both actions deteriorate the state, but safe has lower immediate expected loss. A local comparison therefore favors safe. Near catastrophe, however, the Bellman-optimal policy reverses: the safe action is no refuge because it leads to certain ruin only more slowly, so the agent gambles because a risky win is the only action that can push catastrophe further into the future, and discounting makes deferred catastrophe strictly preferable. We refer to this as the \emph{hail Mary} effect (risk-seeking-near-ruin); it parallels risk seeking in losses from prospect theory. Because both actions have negative expected drift, the absorbing boundary is reached with probability one under any policy \citep{puterman1994markov}; the optimal policy maximizes expected discounted reward by \emph{delaying}---not escaping---ruin. In all 69 decline configurations tested, the states nearest catastrophe are \emph{always} risky.

Unlike in growth, an S-curve does not arise in decline. Instead, $\Vstar$ is a concave staircase (Figure~\ref{fig:decline_example}). Two structural differences from growth explain this pattern.

\paragraph{No convex region.} In growth, convexity arose because the safe zone acted as a discounted approach ramp to the high-value risky zone (equation~\ref{eq:safe_zone}). In decline, the near-boundary zone is instead the hail Mary region, where the agent takes the \emph{risky} action. The risky action's value near catastrophe is dominated by the constant catastrophe penalty: $\Qrisky(S) = p[\Dw + \beta\,\Vstar(S + \Dw)] + (1{-}p)[\Dl + \beta\,\VCAT]$. Since $\VCAT$ is a fixed constant, $\Vstar(S)$ near the boundary varies only through $\Vstar(S + \Dw)$, which changes slowly because $S + \Dw$ sits far from catastrophe. There is no exponential amplification as in the growth safe zone, so no convexity appears.

\paragraph{Staircase at $|\Ds|$ intervals.} Away from the boundary, the agent takes the safe action, which jumps $|\Ds|$ states toward catastrophe. This creates a modular recurrence, $\Vstar(S) = \Ds + \beta\,\Vstar(S - |\Ds|)$, linking every state to the one $|\Ds|$ steps closer. At every multiple of $|\Ds|$, the safe-action destination crosses a policy boundary inherited from the hail Mary region, producing a sharp rise followed by a flat plateau---the staircase. The step height decays by exactly $\beta$ per period, producing overall concavity: the marginal value of distance from catastrophe diminishes as $\beta^k$, so $\Vstar$ has a concave envelope.

\begin{figure}[t]
\centering
\includegraphics[width=0.95\textwidth]{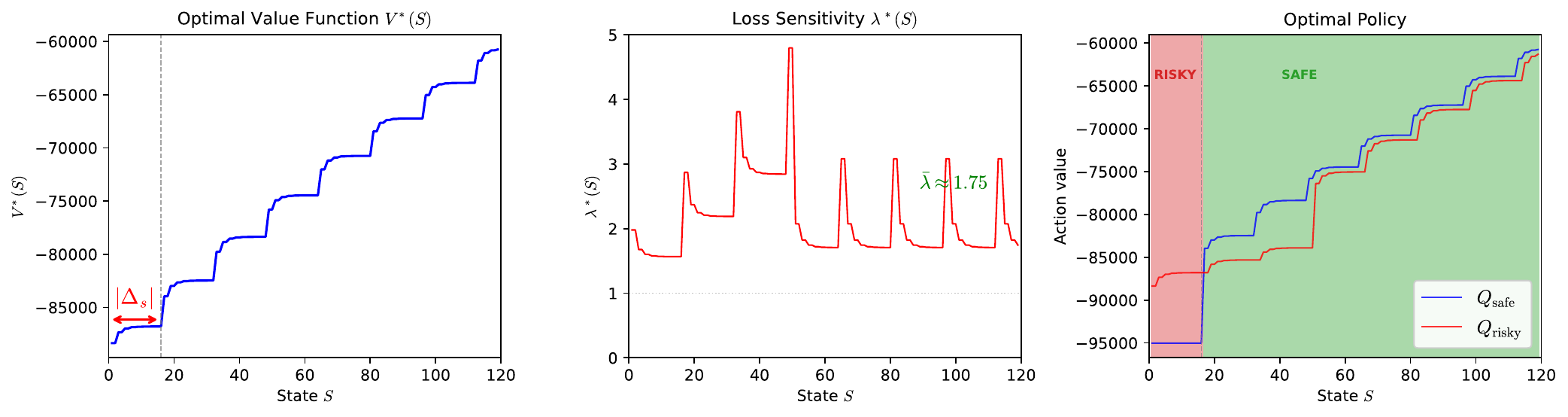}
\caption{\textbf{Decline environment: the hail Mary effect} ($\Ds = -16$, $\Dw = +30$, $\Dl = -50$, $p = 0.40$, $\beta = 0.95$, $\VCAT = -100{,}000$). \emph{Left:} staircase-shaped $\Vstar(S)$. \emph{Center:} periodic $\lambdastar(S)$. \emph{Right:} the agent gambles desperately in the 16 states nearest catastrophe---risk seeking in losses. A local expected-loss comparison would favor safe, yet the Bellman-optimal policy gambles.}
\label{fig:decline_example}
\end{figure}

\subsection{The reflection effect is optimal policy switching}
\label{sec:reflection}

The same Bellman equation and the same risk-neutral objective produce opposite risk attitudes depending on environmental drift, overturning opposite local action rankings in the two regimes:

\begin{center}
\small
\begin{tabular}{lccc}
\toprule
 & Local benchmark & \textbf{Bellman-optimal policy} & PT analogue \\
\midrule
Growth ($\EE[\text{risky}] > \Ds > 0$) & Risky everywhere & \textbf{Safe near boundary} & Risk aversion in gains \\
Decline ($\EE[\text{risky}] < \Ds < 0$) & Safe everywhere & \textbf{Risky near boundary} & Risk seeking in losses \\
\bottomrule
\end{tabular}
\end{center}

This is the reflection effect as a property of continuation values rather than preference curvature.

\section{A closed-form formula for asymptotic loss aversion}
\label{sec:scaling}

The loss sensitivity $\lambdastar(S)$ is not constant: it spikes near the safe-zone boundary and decays to an asymptotic plateau $\bar{\lambda}$ in the far field. We derive a closed-form expression for this terminal value, validate it numerically, and then use it to decompose $\bar{\lambda}$ into a boundary contribution (the symmetric case) and an asymmetry contribution.

\begin{theorem}[Closed-form asymptotic loss aversion]
\label{thm:closedform}
Under the model of Section~\ref{sec:model} and the growth condition $\EE[\text{risky}] > \Ds > 0$, the limit $\bar{\lambda} = \lim_{S \to \infty} \lambdastar(S)$ exists and satisfies
\begin{equation}
\bar{\lambda} \;=\; \frac{z^{-r} - 1}{1 - z}\,,\qquad \beta\bigl[p\,z + (1{-}p)\,z^{-r}\bigr] = 1\,,\qquad z \in (0,1),
\label{eq:master}
\end{equation}
where $r = |\Dl|/\Dw$ and $z$ is the unique root of the characteristic equation in $(0,1)$.
\end{theorem}
\begin{proof}
\emph{Far-field optimality of risky.} As $S \to \infty$, $\Vstar(S) \to V_\infty := \EE[\text{risky}]/(1-\beta)$ (the boundary-free value of always playing risky). Hence $\Qrisky(S) - \Qsafe(S) \to \EE[\text{risky}] - \Ds > 0$ by the growth hypothesis, so risky is uniquely optimal for all $S$ beyond some $S_0$.

\emph{Linear recurrence.} Define $u(S) = V_\infty - \Vstar(S) \geq 0$. For $S \geq S_0$, the Bellman equation reduces to $u(S) = \beta[p\,u(S+\Dw) + (1-p)\,u(S+\Dl)]$, a linear homogeneous recursion with non-negative coefficients.

\emph{Characteristic equation and unique decaying real root.} The exponential ansatz $u(S) = C\rho^S$ yields $f(\rho) := \beta[p\rho^{\Dw} + (1-p)\rho^{\Dl}] = 1$. Each summand $\rho^a$ is convex for $a$ a positive integer or $a < 0$ (so $a(a-1) \geq 0$), hence $f$ is convex on $\rho > 0$; it diverges as $\rho \to 0^+$ and as $\rho \to \infty$, and equals $\beta < 1$ at $\rho = 1$. Therefore $f$ crosses $1$ exactly once in $(0,1)$, at a unique real root $\rho^* \in (0,1)$.

\emph{Dominant-mode asymptotics.} The recursion is a finite-span linear difference equation with strictly positive coefficients. By standard renewal-theoretic / Perron-root asymptotics for such recurrences \citep[e.g.,][Ch.~XI]{feller1971introduction}, any non-negative solution that decays to zero is asymptotically proportional to the minimal positive real root of the characteristic equation. Hence
\begin{equation*}
u(S) \sim C\,(\rho^*)^S
\end{equation*}
for some $C > 0$, where $\rho^* \in (0,1)$ is the unique root identified above.

\emph{Limit existence and value.} From $u(S) \sim C(\rho^*)^S$, the ratios $u(S+\Dw)/u(S) \to (\rho^*)^{\Dw}$ and $u(S+\Dl)/u(S) \to (\rho^*)^{\Dl}$. Substituting into $\lambdastar(S) = (u(S+\Dl) - u(S))/(u(S) - u(S+\Dw))$ yields $\lambdastar(S) \to ((\rho^*)^{\Dl} - 1)/(1 - (\rho^*)^{\Dw})$. The substitution $z = (\rho^*)^{\Dw} \in (0,1)$, with $r = |\Dl|/\Dw$ giving $(\rho^*)^{\Dl} = z^{-r}$, produces the displayed expression.
\end{proof}

\begin{proposition}[Scale invariance]
\label{prop:scale}
$\bar{\lambda}$ depends on $(\Dw, \Dl)$ only through the ratio $r = |\Dl|/\Dw$.
\end{proposition}
\begin{proof}
The substitution $z = \rho^{\Dw}$ rewrites both the characteristic equation and the formula for $\bar{\lambda}$ in Theorem~\ref{thm:closedform} purely in terms of $(p, \beta, r)$.
\end{proof}

\begin{proposition}[Loss aversion always emerges for $r \geq 1$]
\label{prop:lambda1}
Under the assumptions of Theorem~\ref{thm:closedform} and $r \geq 1$, $\bar{\lambda} > 1$ strictly.
\end{proposition}
\begin{proof}
By Theorem~\ref{thm:closedform}, $\bar{\lambda} > 1 \iff z^{-r} - 1 > 1 - z \iff z^{-r} + z > 2$. For $z \in (0,1)$, AM--GM gives $z^{-1} + z > 2$ strictly. For $r > 1$ and $z \in (0,1)$, $z^{-r} > z^{-1}$, so $z^{-r} + z > 2$ a fortiori.
\end{proof}

\paragraph{Validation.} Equation~\eqref{eq:master} achieves $R^2 = 0.999$ against the simulated $\bar{\lambda}$ plateaus from all 343 growth-environment value iteration solutions, with median relative error $0.14\%$. Four controlled experiments (Figure~\ref{fig:sensitivity}) isolate the role of each parameter. \emph{Discount factor $\beta$} (panel a): lower $\beta$ raises $\bar{\lambda}$, because risky-loss steps put catastrophe within one or two periods (its near-term contribution to $\Vstar$ is barely discounted), while risky-win steps yield payoff that accrues over many \emph{future} periods of continued play (all discounted by powers of $\beta$); impatience downweights the gain-side flow faster than the loss-side catastrophe penalty, raising the slope ratio. \emph{Win probability $p$} (panel b), varied via payoff spread at fixed $r$ and $\EE[\text{risky}]$, is the dominant driver: $\bar{\lambda}$ is monotonic in $p$, collapsing toward unity as $p \to 0.5^+$ and diverging as $p \to 1$. \emph{Asymmetry ratio $r$} (panel c) adds to the boundary contribution; even $r < 1$ produces $\bar{\lambda} > 1$ when $p > r/(1+r)$. \emph{Joint payoff scaling} (panel d) leaves $\bar{\lambda}$ invariant, confirming Proposition~\ref{prop:scale}. Probability structure $(p, \beta)$ thus drives $\bar{\lambda}$; payoff magnitude does not.

\begin{figure}[t]
\centering
\includegraphics[width=\textwidth]{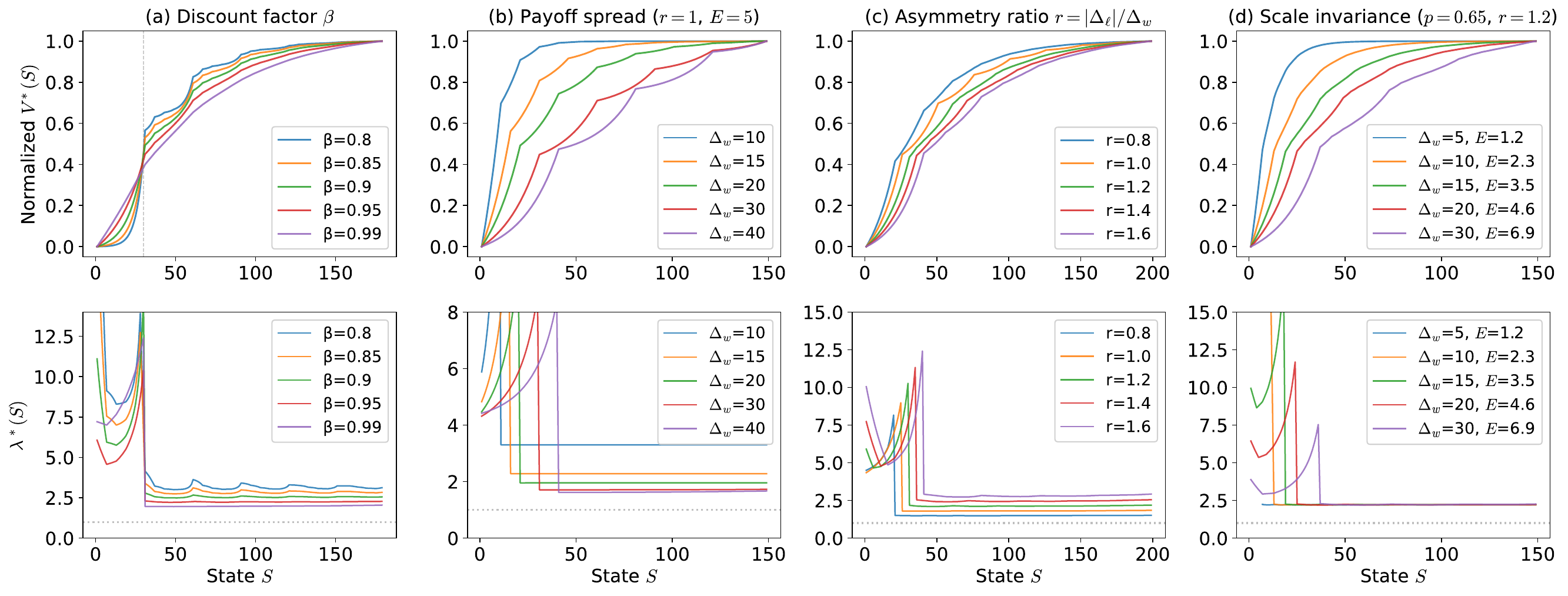}
\caption{\textbf{Parameter sensitivity of $\Vstar(S)$ (top) and $\lambdastar(S)$ (bottom).} Each column isolates one degree of freedom. \textbf{(a)}~Discount factor $\beta$. \textbf{(b)}~Payoff spread at fixed $r = 1$, $\EE[\text{risky}] = 5$: as $p \to 0.5$, $\bar{\lambda}$ falls. \textbf{(c)}~Asymmetry ratio $r$: $\bar{\lambda}$ scales with $r$; even $r < 1$ yields $\bar{\lambda} > 1$. \textbf{(d)}~Scale invariance.}
\label{fig:sensitivity}
\end{figure}

\subsection{Symmetric vs.\ asymmetric gambles}
\label{sec:symasym}

Theorem~\ref{thm:closedform} yields a clean decomposition of $\bar{\lambda}$ into a \emph{boundary contribution}, present even for symmetric payoffs, and an \emph{asymmetry contribution}, introduced by $r > 1$. Define
\begin{equation}
\bar{\lambda}_\text{sym}(p, \beta) := \bar{\lambda}(p, \beta, r{=}1), \qquad
\sigma(p, \beta, r) := \frac{\bar{\lambda}(p, \beta, r) - \bar{\lambda}_\text{sym}(p, \beta)}{\bar{\lambda}(p, \beta, r) - 1},
\label{eq:asym_share}
\end{equation}
the asymmetric plateau and the asymmetry share, respectively. The numerator is the increment from the symmetric to the asymmetric case at fixed $(p, \beta)$; the denominator normalizes by the gap above unity, so $\sigma$ measures the fraction of $\bar{\lambda} - 1$ attributable to $r > 1$.

\paragraph{Symmetric case (boundary alone).}
For $r = 1$, eq.~\eqref{eq:master} reduces to $\bar{\lambda}_\text{sym} = 1/z^*$, where $z^* \in (0,1)$ solves $\beta[pz + (1{-}p)/z] = 1$. By Proposition~\ref{prop:lambda1}, $\bar{\lambda}_\text{sym} > 1$ throughout the symmetric growth regime ($p > 1/2$). Worked values at $\beta = 0.95$: $1.55$ at $p = 0.55$, $2.14$ at $p = 0.65$, $4.34$ at $p = 0.80$, diverging as $p \to 1$. So an absorbing boundary applied to a perfectly symmetric risky/safe choice already produces a slope ratio of similar order to the empirical prospect-theory estimate $\lambda \approx 2.25$ \citep{tversky1992advances}---no preference or payoff asymmetry required.

\paragraph{Asymmetric case (boundary plus asymmetry).}
For $r > 1$, $\bar{\lambda}$ rises further. Table~\ref{tab:asym_decomp} reports descriptive statistics for $\sigma$ swept over $p \in \{0.55, 0.60, 0.65, 0.70, 0.75, 0.80\}$ and $\beta \in \{0.85, 0.90, 0.95, 0.99\}$ at three asymmetry levels (growth-feasible cells only).

\begin{table}[h]
\centering
\caption{Asymmetry share $\sigma$ (eq.~\ref{eq:asym_share}) across the $(p, \beta)$ grid, by asymmetry level $r$.}
\label{tab:asym_decomp}
\small
\begin{tabular}{cccccc}
\toprule
$r$ & $n$ & median $\sigma$ & mean $\sigma$ & IQR & range \\
\midrule
$1.25$ & $20$ & $4.6\%$ & $5.8\%$ & $[2.6\%, 8.7\%]$ & $[0.4\%, 13.5\%]$ \\
$1.50$ & $16$ & $8.3\%$ & $8.9\%$ & $[5.1\%, 12.8\%]$ & $[0.9\%, 18.6\%]$ \\
$2.00$ & $12$ & $13.9\%$ & $14.7\%$ & $[10.1\%, 20.4\%]$ & $[2.2\%, 26.5\%]$ \\
\bottomrule
\end{tabular}
\end{table}

In every one of the $48$ $(p, \beta, r)$ cells tested, the boundary contribution $\bar{\lambda}_\text{sym} - 1$ exceeds the asymmetry contribution $\bar{\lambda} - \bar{\lambda}_\text{sym}$. The asymmetry share grows monotonically with $r$ but stays below $30\%$ even at $r = 2$. The boundary therefore dominates $\bar{\lambda}$ across realistic parameters and across small to large asymmetry.

\paragraph{Connection to behavioral-economics empirics.}
This decomposition matches an empirical pattern. \citet{walasek2015makeloss} show experimentally that estimated $\lambda$ can be tuned from below unity to well above two by manipulating the symmetry of the choice set, and meta-analyses by \citet{brown2024meta} and \citet{yechiam2025loss} document the same dependence in observational data. Equation \eqref{eq:master} implies that, in any sequential setting with absorbing failure, effective loss aversion has a boundary-induced component determined by $(p, \beta)$ to which payoff asymmetry $r$ adds further; controlled symmetric gambles isolate the boundary, and asymmetric gambles measure the additional contribution.

\section{Robustness}
\label{sec:validation}

The theoretical claims so far were established under exact value iteration with deterministic transitions. We now show that they survive (i)~a model-free agent with no knowledge of the transition kernel, and (ii)~stochastic transitions whose distribution and channel of action we vary.

\subsection{Tabular Q-learning}
\label{sec:qlearning}

A tabular Q-learning agent ($\epsilon$-greedy with constant $\epsilon = 0.5$; learning rate $\alpha = 0.01$; $5 \times 10^{5}$ episodes; max $2{,}000$ steps per episode; optimistic Q-initialization at $\max(\EE[\text{risky}], \Ds)/(1-\beta)$; uniform random start state; episode terminates on entering $\SCAT$) with no knowledge of the transition model reproduces all qualitative features (Figure~\ref{fig:qlearning}). On the canonical Figure~\ref{fig:growth_example} growth and Figure~\ref{fig:decline_example} decline configurations with $\bar{S} = 1000$, Pearson correlation between $V^Q$ and $\Vstar$ is $0.980$ in growth and $1.000$ in decline. The S-shape, safe zone, hail Mary region, and $\lambdastar > 1$ all emerge from learning rather than from exact computation. We treat Q-learning as a qualitative reproduction of the mechanism; the closed-form formula~\eqref{eq:master} remains the recommended estimator when the transition model is known. The reproduction is bundled in the supplementary archive (\texttt{code/run\_qlearning\_figure.py}); each configuration runs in approximately $30$ seconds on a single CPU.

\begin{figure}[t]
\centering
\includegraphics[width=\textwidth]{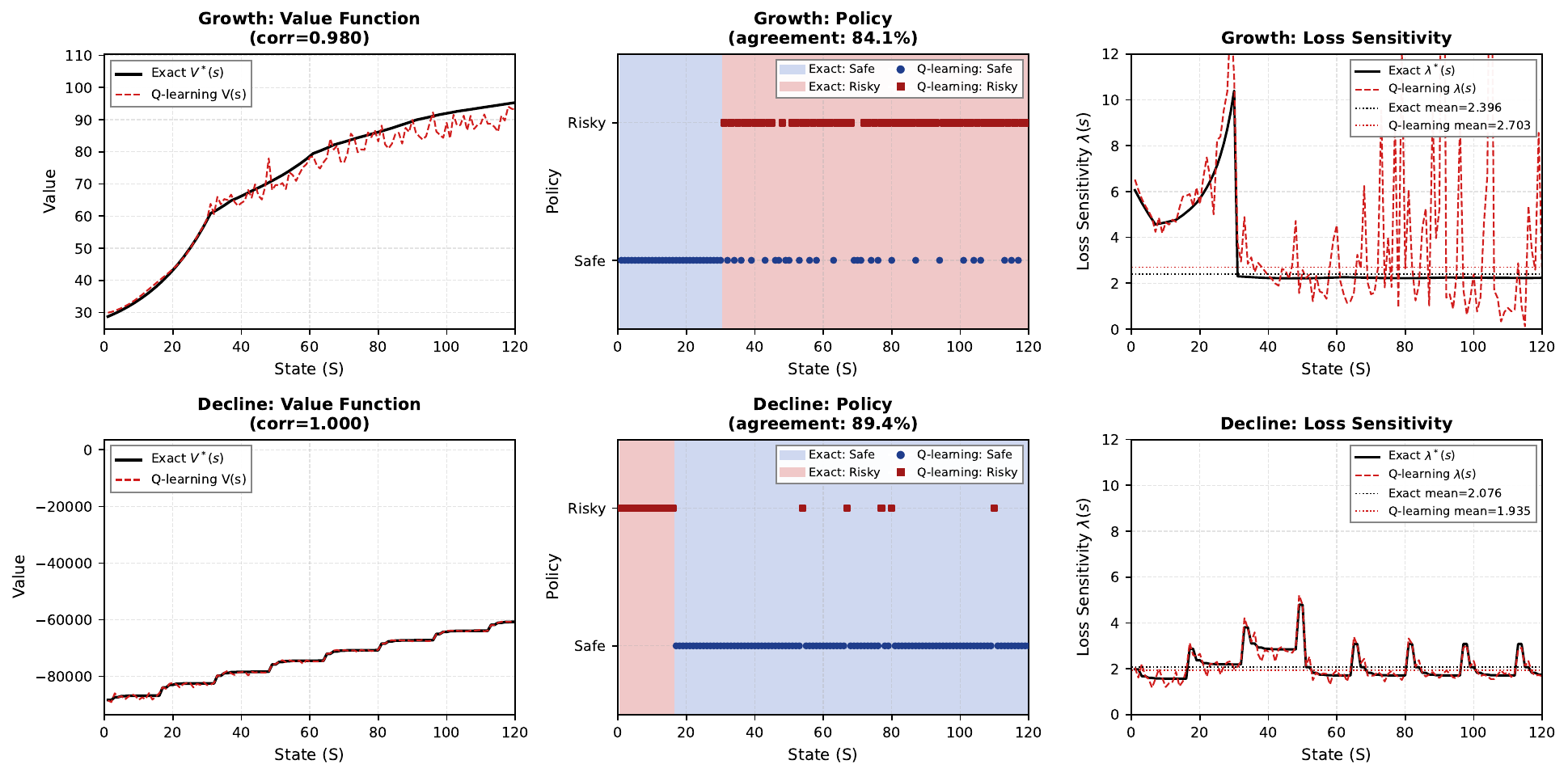}
\caption{\textbf{Q-learning reproduces all phenomena.} \emph{Top:} growth (correlation $0.980$). \emph{Bottom:} decline (correlation $1.000$). A model-free agent discovers the same S-shape, loss aversion, and policy structure.}
\label{fig:qlearning}
\end{figure}

\subsection{Stochastic transitions}
\label{sec:noise}

To check that the prospect-theory signatures do not depend on the deterministic transition model in \eqref{eq:bellman}, we perturb each transition's destination by additive integer-valued noise on the asymmetric baseline ($p = 0.65$, $\beta = 0.95$, $r = 1.25$, $\Dw = 8$, $\Dl = -10$, $\Ds = 1$, $\VCAT = -100$; closed-form prediction $\bar\lambda_\text{cf} \approx 2.23$). The integer state space is preserved by distributing probability mass over neighboring states: for nominal destination $S + \Delta_a + \xi$ with $\xi \sim F_\sigma$, the discrete transition kernel is $P(S'') = F_\sigma(S'' + 0.5) - F_\sigma(S'' - 0.5)$, truncated and renormalized. We sweep three noise families---Gaussian $\mathcal{N}(0,\sigma^2)$, heavy-tailed Student-$t_3$ rescaled to match Gaussian variance, and asymmetric skew-normal ($\alpha = 2$, variance-matched, zero-mean)---applied to (a)~the safe action only, (b)~risky only, or (c)~both, with $\sigma \in \{0, 0.10, 0.25, 0.50\}\cdot|\Dl|$.

Table~\ref{tab:noise} reports the maximum relative deviation of $\hat{\bar{\lambda}}$ from the noiseless closed-form prediction across the $\sigma$ sweep. Three patterns hold across all three noise distributions: (1)~safe-channel noise alone leaves $\bar{\lambda}$ essentially unchanged ($\leq 0.41\%$ even at $\sigma = 0.5|\Dl|$); (2)~risky-channel noise (alone or combined with safe) reduces $\bar{\lambda}$ monotonically, by at most $9.6\%$ at $\sigma = 0.5|\Dl|$; (3)~the qualitative S-shape, hail Mary region, and policy boundary are unchanged at every noise level. Qualitative behavior is invariant to the noise distribution; the quantitative plateau varies modestly (within $\approx 10\%$ at the largest noise level), and the difference between Gaussian, heavy-tailed, and asymmetric noise is small.

\begin{table}[h]
\centering
\caption{Maximum relative deviation of $\hat{\bar{\lambda}}$ from the closed-form prediction $\bar{\lambda}_\text{cf} \approx 2.23$ across $\sigma \in \{0, 0.10, 0.25, 0.50\}\cdot|\Dl|$. Asymmetric baseline ($p = 0.65$, $\beta = 0.95$, $r = 1.25$).}
\label{tab:noise}
\small
\begin{tabular}{lccc}
\toprule
Noise distribution & Safe only & Risky only & Both \\
\midrule
Gaussian                                    & $0.36\%$ & $\phantom{0}9.54\%$ & $\phantom{0}9.57\%$ \\
Student-$t_3$ (heavy-tailed)                & $0.41\%$ & $\phantom{0}9.00\%$ & $\phantom{0}8.98\%$ \\
Skew-normal (asymmetric, $\alpha = 2$)      & $0.38\%$ & $\phantom{0}9.35\%$ & $\phantom{0}9.30\%$ \\
\bottomrule
\end{tabular}
\end{table}

\section{Discussion}
\label{sec:discussion}

We have shown that an absorbing catastrophe boundary is structurally sufficient for three prospect-theory-like signatures to emerge under standard Bellman optimality, with linear rewards and no preference asymmetry. The claim is one of sufficiency, not exclusivity: probability weighting, framing, and reference dependence may still play independent roles in human behavior, and we do not argue that observed human loss aversion is itself optimal. The mechanism extends from literal catastrophes (bankruptcy, extinction) to effectively absorbing psychological or institutional thresholds \citep{koszegi2006model}. The decomposition in Section~\ref{sec:symasym} aligns with empirical evidence that measured loss aversion varies with payoff asymmetry \citep{brown2024meta, yechiam2025loss, walasek2015makeloss, gal2018loss, yechiam2013losses, mukherjee2017magnitude}: $\bar\lambda$ has a boundary-induced component determined by $(p, \beta)$ and an asymmetry component that depends on $r$, rather than a fixed psychological constant. For reinforcement learning, the result implies that conservative or desperate behavior near absorbing failure can emerge from continuation values alone, with no need to hand-code ``safe'' behavior into the agent's preferences.

\paragraph{Limitations and broader impact.} The model is one-dimensional with binary actions, known transitions, and a single boundary; probability weighting, framing, and the endowment effect are outside its scope, and the analogy between sequential MDP decisions and one-shot gambles is structural rather than literal. Empirical validation requires environments whose forward-looking drift can be experimentally manipulated; we leave this and extensions to higher dimensions, continuous actions, partial observability, and deep RL to future work. Recognizing that risk-seeking or risk-averse behavior near irreversible failure can be rational rather than pathological may inform safety-constraint design for RL systems and help interpret loss-averse human choices in high-stakes settings; we see no direct negative societal impact specific to this work.

\bibliographystyle{plainnat}

\newpage
\appendix

\section{Complete parameter summary}
\label{app:params}

\begin{table}[h]
\caption{Organization of 495 configurations across 14 experimental series.}
\label{tab:configs}
\centering
\small
\begin{tabular}{@{}llccl@{}}
\toprule
Series & Description & Qty & Key variables & Primary finding \\
\midrule
1 & Ratio sweep & 15 & $|\Dl/\Dw|$ & $\lambdastar$ scales with $r$ \\
2 & Symmetric payoffs & 15 & $r=1.0$, $\beta$, $E$ & Loss aversion without asymmetry \\
3A--C & Beta sweep & 57 & $\beta$ & Monotone decrease \\
4 & $\VCAT$ locality & 12 & $\VCAT$ & Zero effect on plateau \\
5 & State space & 8 & $\bar{S}$ & Stable for $\bar{S} \geq 200$ \\
7 & Drift sweep & 10 & $\Ds$ & Phase transition \\
8 & Phase transition & 15 & $\Ds = \EE[\text{risky}]$ & Sharp discontinuity \\
9--10 & Loss/Win sweep & 30 & $\Dl$, $\Dw$ & Validates ratio scaling \\
11 & Probability sweep & 15 & $p$ & Via expected payoff \\
17 & Heatmaps & 175 & $\beta \times r$, $\Ds \times r$ & Interaction effects \\
18 & $\EE[\text{risky}]$ sweep & 33 & $\EE[\text{risky}]$ & Secondary effect \\
19--20 & Interaction & 110 & $E \times r$, $E \times \beta$ & Separable contributions \\
\midrule
\textbf{Total} & -- & \textbf{495} & -- & -- \\
\bottomrule
\end{tabular}
\end{table}

\section{Value iteration algorithm}
\label{app:algorithm}

\noindent\textbf{Algorithm 1:} MDP value iteration.

\medskip
\noindent\fbox{\parbox{0.96\textwidth}{
\small
\textbf{Input:} $(\SCAT, \bar{S}, \VCAT, \Ds, \Dw, \Dl, p, \beta, \epsilon)$\\
\textbf{Output:} $\Vstar$, $\pi^*$, $\lambdastar$

\medskip
1.\ Initialize $V(S) \leftarrow 0$ for $S > \SCAT$;\; $V(\SCAT) \leftarrow \VCAT$.

2.\ \textbf{Repeat:}\\
\hspace*{1em}$V_{\text{old}} \leftarrow V$\\
\hspace*{1em}\textbf{For} $S = \SCAT + 1$ \textbf{to} $\bar{S}$:\\
\hspace*{2em}$\Qsafe \leftarrow \Ds + \beta \cdot V_{\text{old}}(\text{clamp}(S + \Ds, \SCAT, \bar{S}))$\\
\hspace*{2em}$\Qrisky \leftarrow p[\Dw + \beta \cdot V_{\text{old}}(\text{clamp}(S + \Dw, \SCAT, \bar{S}))] + (1{-}p)[\Dl + \beta \cdot V_{\text{old}}(\text{clamp}(S + \Dl, \SCAT, \bar{S}))]$\\
\hspace*{2em}$V(S) \leftarrow \max(\Qsafe, \Qrisky)$\\
\hspace*{1em}\textbf{Until} $\|V - V_{\text{old}}\|_\infty < \epsilon$.

3.\ Extract $\pi^*(S) = \arg\max(\Qsafe, \Qrisky)$ and $\lambdastar(S)$ via \eqref{eq:lambda}.

\medskip
where $\text{clamp}(x, \text{lo}, \text{hi}) = \min(\max(x, \text{lo}), \text{hi})$.
}}

\section{Additional figures}
\label{app:additional_figures}

\begin{figure}[h]
\centering
\includegraphics[width=\textwidth]{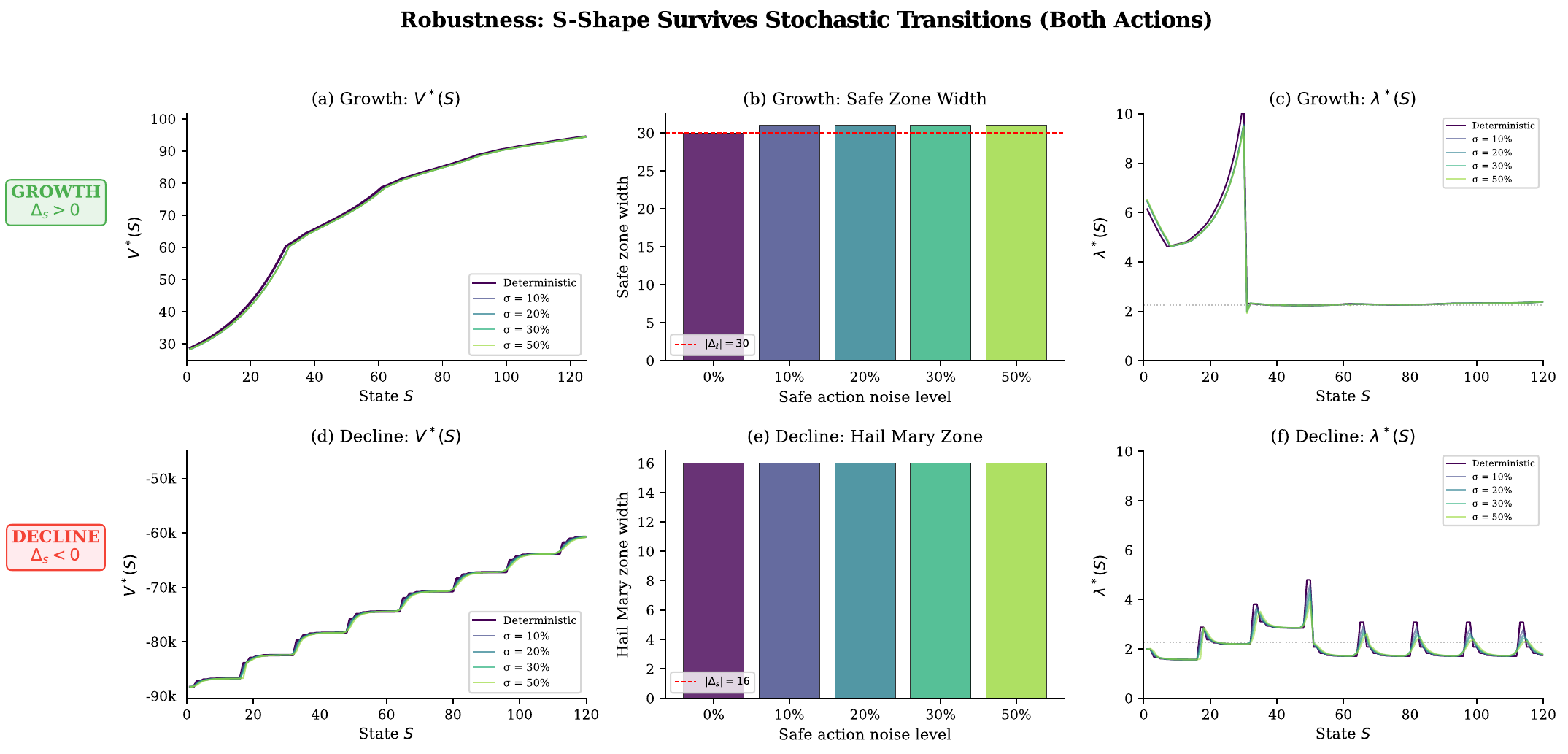}
\caption{\textbf{Robustness to stochastic transitions.} Safe-zone width, hail Mary zone width, and $\lambdastar$ are virtually unchanged from $0\%$ to $50\%$ Gaussian noise on transitions. Section~\ref{sec:noise} extends this analysis to heavy-tailed and asymmetric noise distributions and to per-channel sweeps; the quantitative summary is in Table~\ref{tab:noise}.}
\label{fig:robustness}
\end{figure}

\end{document}